\documentclass[11pt]{article}
\usepackage[utf8]{inputenc} 
\usepackage[T1]{fontenc}    

\usepackage{amsmath,amsthm}
\usepackage{subcaption}

\makeatletter
\let\rel@kern\relax 
\makeatother

\usepackage{newtxtext}
\usepackage{newtxmath}
\usepackage{bm}
\usepackage{natbib}

\usepackage[utf8]{inputenc} 

\usepackage{subcaption}
\usepackage{float}

\usepackage{bm}

\usepackage{graphicx}%
\usepackage{multirow}%

\usepackage{mathrsfs}%
\usepackage[title]{appendix}%
\usepackage{manyfoot}%
\usepackage{booktabs}%
\usepackage[ruled,vlined]{algorithm2e}
\usepackage{listings}%
\usepackage{mathtools}

\usepackage{rotating}
\usepackage{comment}
\usepackage{soul}

\newcommand{\mathbbm}{\mathbb}

\newcommand{\Eb}{\mathbbm{E}}
\newcommand{\Rb}{\mathbbm{R}}

\newcommand{\Pb}{\mathbbm{P}}
\newcommand{\Qb}{\mathbbm{Q}}
\newcommand{\Zb}{\mathbbm{Z}}

\newcommand{\Ac}{\mathcal{A}}
\newcommand{\Dc}{\mathcal{D}}
\newcommand{\Ec}{\mathcal{E}}

\newcommand{\Qc}{\mathcal{Q}}

\newcommand{\Bc}{\mathcal{B}}

\newcommand{\Xc}{\mathcal{X}}

\newcommand{\Rc}{\mathcal{R}}

\newcommand{\Uc}{\mathcal{U}}
\newcommand{\Tc}{\mathcal{T}}

\newcommand{\Pc}{\mathcal{P}}

\newcommand{\Vc}{\mathcal{V}}

\newcommand{\Lc}{\mathcal{L}}

\newcommand{\argmin}{\mathop{\text{argmin}}}

\renewcommand{\Omega}{\varOmega}
\newcommand{\D}{\textup{d}}
\newcommand{\Gc}{\mathcal{G}}

\newcommand{\Cdot}{\,\cdot\,}
\newcommand{\1}{\bm{1}}
\newtheorem{theorem}{Theorem}[section]
\newtheorem{proposition}[theorem]{Proposition}

\newtheorem{remark}[theorem]{Remark}

\DeclareMathOperator{\avar}{{AV\hspace{-0.1em}{\scriptstyle{@}}\hspace{-0.05em}R}}
\newcommand{\eqdef}{\mathrel{\overset{\raisebox{-0.02em}{$\scriptstyle\vartriangle$}}{\,=\,}}}

\DeclareMathOperator*{\esssup}{ess\,sup}
\newcommand{\icx}{\mathrel{\preceq_{\textup{icx}}}}
\newcommand{\Gini}{\varGamma}

\newenvironment{tightitemize}{%
    \list{{\textup{$\bullet$}}}{\settowidth\labelwidth{{\textup{\qquad}}}
    \leftmargin\labelwidth \advance\leftmargin\labelsep
    \parsep 0pt plus 1pt minus 1pt \topsep 3pt \itemsep 3pt
    }}{\endlist}

\newenvironment{tightlist}[1]{%
    \list{{\textup{(\roman{enumi})}}}{\settowidth\labelwidth{{\textup{(#1)}}}
    \leftmargin 0pt \advance\leftmargin\labelsep \itemindent \parindent
    \parsep 0pt plus 1pt minus 1pt \topsep 0pt \itemsep 0pt
    \usecounter{enumi}}}{\endlist}

\usepackage{xcolor}
\definecolor{plum}{rgb}{0.3,0,0.7}

\begin{document}

\title{Mini-Batch Risk-Averse Deep Q-Learning:\\ A Robot Navigation Case Study}


\author{ {Aayush} {Patel} and Andrzej Ruszczy\'nski \footnote{Department of Management Science and Information Systems, Rutgers University, email: ap2475@scarletmail.rutgers.edu; rusz@rutgers.edu} 
       }

\date{September 7, 2026}
%

\maketitle



\abstract{
We study the control of Markov decision processes in which the quality of a
policy is evaluated by a dynamic, time-consistent Markov risk measure rather
than by an expected discounted cost. The main obstacle to combining such
measures with reinforcement learning is that a transition risk mapping depends
on the transition kernel in a nonlinear way, and therefore cannot be estimated
from a single observed transition. We remove this obstacle by employing
mini-batch transition risk mappings: the mapping is applied to the
empirical measure of $N$ independent next-state samples, and the result is
averaged. The resulting mapping is again coherent. However, as an
expected value of a function of $N$ next-state values, it admits
an unbiased one-sample estimator. 

We embed this mapping into a double deep Q-network, analyze the two sources of
estimation bias that arise, and obtain a risk-averse Q-learning method
applicable to state spaces far beyond the reach of tabular schemes. The method
is applied to an underwater robot navigation problem, in which a vehicle must
visit collection points, gather stochastic information payloads, and deliver
them at transmission points, while exposed at each step to the risk of
destruction. A hierarchical decomposition delegates path execution to an exact
graph search and confines learning to the high-level ``collect or transmit''
decision. A low-dimensional feature map, invariant under the symmetries of
the problem, replaces the raw state--configuration encoding. In experiments on
$300$ held-out environments, the resulting policies transfer to instance sizes
never seen in training, and already $N=2$ reduces the upper semideviation of
the outcome distribution while simultaneously improving its mean whenever the
simulator is misspecified---an empirical counterpart of the duality between
coherent risk measures and distributional robustness.\\[2pt]
}

\section{Introduction}

We consider a Markov decision process (MDP) operating in time $t=0,1,2\dots$,  with a finite state space $\Xc=\{1,\dots,n\}$, which may be very large, and
the action space $\Uc$, which is finite as well and practically possible to enumerate.  The  {feasible control set} is  $U: \Xc\rightrightarrows \Uc$,
and the controlled transition kernel is $\Pb:\Xc\times\Uc \to \Pc(\Xc)$, where $\Pc(\Xc)$ denotes the set of probability distributions on $\Xc$.
 We use $\pi_t$ to denote the \emph{decision rule} to choose the control $u_t$ at time $t$. In general, $\pi_t$ may be a function of $(x_0,x_1,\dots,x_t)$, the states visited
at times $0,1,\dots,t$, and produce a probability distribution on $U(x_t)$, but we shall be mainly concerned with \emph{stationary deterministic Markov policies},
in which $u_t = \pi(x_t)$ with a stationary { (time-invariant)} decision rule~$\pi$.

For any stationary deterministic Markov policy $\varPi=\{\pi,\pi,\dots\}$, and any initial state $x_0$, the sequence of states $\{x_t\}_{t=0,1,\dots}$ is a Markov chain, with the
transition probability matrix $\Pb^\pi$ having rows $\Pb^\pi_{x} = \Pb(x,\pi(x))$, $x\in \Xc$. At each time $t=0,1,\dots$, if the state is $x_t$ and the control is $u_t$, a cost $c(x_t,u_t)$ is incurred, where $c:\Xc\times\Uc\to\Rb$. Thus, under a Markov policy $\varPi$, the resulting sequence of costs is
$c^\pi(x_t) = c(x_t,\pi(x_t))$, $t=0,1,\dots$.
In standard formulations (see, \emph{e.g.},
\cite{Puterman1994,bertsekas2017dynamic}), the objective is to find a control policy that minimizes or maximizes the expected (discounted) sum or the expected average of stage-wise costs {or rewards} over a finite or infinite horizon:
\begin{equation}
\label{V-expected}
v^\pi(x) = \Eb\Big\{ \sum_{t=0}^\infty \alpha^{t} c^\pi(x_t) \,\Big|\, x_0=x\Big\}, \quad \alpha \in (0,1).
\end{equation}
Instead of the standard expected-value objective, we are considering a risk-averse model with a dynamic risk measure,
\begin{equation}
\label{V-rho}
V^\pi(x) = \rho_{0,\infty}\Big(   \big\{ c^\pi(x_t)\big\}_{t=0}^{\infty}\Big), \quad x_0=x,
\end{equation}
used to evaluate the quality of the control policy. The operator $\rho_{0,\infty}(\Cdot)$ assigns 
a real number representing a risk-adjusted total cost to a sequence of random costs.
The objective is to minimize \eqref{V-rho} over the space of all feasible policies. Our choice of the minimization formulation is motivated by the convenience of dealing with risk measures and their convexity, but all our results can be easily translated to the maximization setting, with concave risk measures.

Our plan is to use mini-batch Markov risk measures, recently explored in \cite{ruszczynski2025risk} combined with Q-factor approximations, allowing for a very large state space. While \cite{ruszczynski2025risk} focused on linear value function approximations, we will explore the use of deep Q-networks to find a risk-averse control policy.

Several works introduce models of risk into reinforcement learning: exponential utility functions \cite{Borkar2001, Borkar2002,basu2008learning,fei2022cascaded} and  mean-variance models
\cite{Tamar2012, Prashanth2014}. A few later studies propose heuristic approaches involving specific coherent risk measures, such as
CVaR in the objective or constraints \cite{Chow2014,Chow2015a,ma2018risk}.
Generative model value iteration with coherent measures was analyzed by \citet{yu2018approximate}. Risk-aware Q-learning with Markov risk measures is considered by \citet{huang2017risk}.  All these methods apply to problems with a small number of state-action pairs allowing exhaustive experimentation.

Value function approximations in the context of  distributionally robust MDPs were considered  by \citet{Tamar2014}.
\citet{Tamar2017} study the policy gradient approach for Markov risk measures and use it in an actor--critic type algorithm. Both approaches are heuristic. Policy evaluation with linear architecture and Markov risk measures by a method of temporal differences was analyzed by \citet{kose2021risk}, and asymptotic convergence was proved.  The recent work of \citet{yin2022near} uses variance to control the value iteration procedure.  Refs. \cite{fei2020risk,fei2021risk} focus exclusively on the entropic risk measure, but derive theoretically sound and near-optimal regret bounds.
Recently, \citet{lamrisk} proposed a version
of a risk-aware reinforcement learning method with coherent risk measures and function approximation. However, at each iteration, it uses extensive experimentation (double sampling)  to statistically estimate the risk with high accuracy and high probability at each state-control pair.

Our contribution is threefold.
First, on the modeling side, we identify the structure of the mini-batch
transition risk mapping generated by the worst-case base mapping: it is a
distortion risk measure with a concave distortion, equivalently a spectral
measure belonging to the extended Gini family, whose Kusuoka representation is
an explicit Beta mixture of Average Values-at-Risk
(Proposition~\ref{prop:structure}). In particular, the family is increasing in
the batch size $N$, it converges to the worst-case mapping, and each of its
members is consistent with second-order stochastic dominance. The batch size
$N$ is thus not merely an algorithmic parameter; it indexes a one-parameter
family of coherent measures with a transparent economic meaning, and it is
this structure that makes a single sample sufficient for an unbiased estimate.

Second, on the algorithmic side, we combine the mini-batch mapping with a
double deep Q-network. We show that the estimation bias splits into an
\emph{action-selection} bias, which the double-network construction removes,
and a residual bias caused by the propagation of approximation noise through a
convex mapping, which has the opposite sign
(Remark~\ref{rem:state-max-bias}). The resulting scheme requires no
double sampling and no accurate pointwise risk estimation at each state-control
pair.

Third, on the computational side, we solve a risk-averse underwater robot
navigation problem whose state space is far beyond the reach of tabular
methods. Two devices prove essential: a hierarchical decomposition, which
delegates the deterministic path-execution subproblem to an exact graph search,
and a feature map invariant under the symmetry group of the problem, which
allows a single trained network to be applied to instances of a size never seen
in training.

The paper is organized as follows. Section~\ref{s:Markov_risk} recalls Markov
risk measures and risk-averse dynamic programming. Section~\ref{sec:mini-batch}
introduces the mini-batch transition risk mappings and establishes their
structure. Section~\ref{sec:ddqn} develops the risk-averse deep Q-learning
method. Sections~\ref{s:robot} and~\ref{s:results} present the robot navigation
model and the numerical results, and Section~\ref{s:conclusions} concludes.

\section{Markov Risk Measures and Risk-Averse Dynamic Programming}
\label{s:Markov_risk}

 In order to obtain a dynamic programming formulation, we focus on
\emph{Markov risk measures}, introduced in \cite{Ruszczynski2010Markov} and further developed in \cite{fan2022process,dentcheva2024risk}. 
Denote by $\Vc \eqdef \Rb^{|\Xc|}$ the space of functions of the state.  
The key role in this theory is played by a \emph{transition risk mapping}
 $\sigma:\Xc \times  \Pc(\Xc)\times \Vc\to \Rb$ which fully characterizes such a risk measure and allows for the formulation of the \emph{policy evaluation equation}:
 \begin{equation}
\label{DP-risk-infinite}
V^\pi(x)  = c(x,\pi(x))  + \alpha \sigma\big(x,\Pb(x,\pi(x)), V^\pi\big),\quad x \in \Xc.
\end{equation}

In such an MDP, a Markovian optimal policy $\pi^\star$ exists, and the corresponding optimal
value function $V^\star(\Cdot)$ satisfies the risk-averse \emph{dynamic programming equation}:
 \begin{equation}
\label{DP-risk-infinite-optimal}
\begin{aligned}
V^\star(x) &=  \min_{u\in \Uc(x)} \Big\{c(x,u) +  \alpha \sigma\big(x,\Pb(x,u), V^\star\big)\Big\},\quad
  x\in \Xc.
\end{aligned}
\end{equation}
The policy $\pi^\star$ is given by the minimizers in \eqref{DP-risk-infinite-optimal}.

The Q-factors  for a risk-averse model are defined in a way similar to the risk-neutral case:
\begin{align*}
Q^\pi(x,u)&=c(x,u)+\alpha \sigma\big(x,\Pb(x,u),V^\pi\big),\\
Q^\star(x,u)&=c(x,u)+\alpha \sigma\big(x,\Pb(x,u),V^\star\big);
\end{align*}
therefore,
$V^\pi(x)=Q^\pi(x,\pi(x))$ and
$V^\star(x)=\min_{u\in \Uc(x)} Q^\star(x,u)$. The Q-factors are elements of the space $\Qc \eqdef \Rb^{|\Xc|\times |\Uc|}$.

A simple example of a transition risk  mapping is the conditional expected value:
\begin{equation}
\label{E-form}
\sigma(x,\Pb,V) = \Pb V.
\end{equation}
Its use reduces \eqref{DP-risk-infinite}--\eqref{DP-risk-infinite-optimal} to the standard expected-value dynamic programming.

More interesting are the risk mappings that depend on the transition kernel and the value function in a nonlinear way. An
example is the \emph{mean--semideviation model} \cite{OgryczakRuszczynski1999}:
\begin{equation}
\label{msd-form}
\sigma(x,\Pb,V) = \Pb V
+ \kappa \sum_{y\in \Xc} \Pb(y)\max\big( 0, V(y) -  \Pb V\big), \quad \kappa\in [0,1].
\end{equation}
Another example, rarely used in risk-measure theory due to its conservative nature, is the \emph{worst-case} mapping:
\begin{equation}
\label{sup-form}
\sigma(x,\Pb,V) =  \inf \big\{ b\in \Rb: \Pb[V \le b ] = 1 \big\} = \esssup_{\Pb} V .
\end{equation}
All examples above, as functionals of the value function $V(\Cdot)$, with a fixed $\Pb=\Pb(x,u)$, satisfy the axioms of a coherent measure of risk \cite{ArtznerDelbaenEberEtAl1999}:
\begin{tightitemize}
\item[\textbf{Convexity:}] $\sigma(x,\Pb,\lambda V + (1-\lambda) W) \leq  \lambda \sigma(x,\Pb,V) + (1-\lambda) \sigma(x,\Pb,W)$,  $\forall\, \lambda \in [0,1]$;
\item[\textbf{Monotonicity:}] If $V \leq W$ then $\sigma(x,\Pb,V) \leq \sigma(x,\Pb,W)$;
\item[\textbf{Translation:}] $\sigma(x,\Pb,V + a\1) = \sigma(x,\Pb,V) + a$, for all $a \in \Rb$ (here, $\1$ is the vector of 1's);
\item[\textbf{Positive homogeneity:}] $\sigma(x,\Pb,s V) = s\, \sigma(x,\Pb,V)$, for all \mbox{$s \geq 0$}.
\end{tightitemize}
For such transition risk mappings, the dual representation is valid:
\begin{equation}
\label{dual}
\sigma(x,\Pb,V) = \max \Big\{  \Qb V: {\frac{\D \Qb}{\D\Pb} \in \Ac}\Big\},
\end{equation}
with some convex and closed set $\Ac$ of densities; see \cite{RuszczynskiShapiro2006a} and the comprehensive exposition in \cite[Ch. 2]{dentcheva2024risk}. The set $\Ac$ depends on $\Pb$, in general.
Formula \eqref{dual} is the precise sense in which risk aversion and
distributional robustness are two descriptions of the same object: a coherent
transition risk mapping is a worst case over an ambiguity set of kernels, which
is the model postulated in the robust dynamic programming of
\cite{iyengar2005robust,nilim2005robust}. We shall return to this reading when
interpreting the numerical results in Section~\ref{s:results}.

Monotonicity and translation equivariance imply that every coherent transition
risk mapping is nonexpansive in the supremum norm,
\begin{equation}
\label{nonexpansive}
\big| \sigma(x,\Pb,V) - \sigma(x,\Pb,W)\big| \le \|V-W\|_\infty, \quad V,W \in \Vc,
\end{equation}
so the operators on the right-hand sides of \eqref{DP-risk-infinite} and
\eqref{DP-risk-infinite-optimal} are contractions with modulus $\alpha$; this is
the reason why these equations have unique solutions and why the corresponding
value and policy iteration methods converge \cite{Ruszczynski2010Markov,ruszczynski2014erratum}.

Convex combinations of coherent measures of risk are coherent. In particular, we may mix the worst case mapping with the expected value.

\section{Mini-Batch Transition Risk Mappings}
\label{sec:mini-batch}

Because of the nonlinear dependence of the transition risk mappings on the probability measure, their statistical estimation is challenging.
To address this issue,  we adapt to our case the construction of \emph{mini-batch risk forms} from \cite{dentcheva2023mini}. For
a sample $Y^{1:N}=(Y^1,\dots,Y^N)$ with $N$ independent observations of the next state distributed according to $\Pb(x,u)$ in $\Xc$, we consider a random empirical measure
$\Pb^{(N)}(x,u)=\frac{1}{N}\sum_{j=1}^N\delta_{Y^j}$ (each $\delta_{Y^j}$ is the unit mass at $Y^j$). Then, for a fixed coherent
transition risk mapping $\sigma^{\text{b}}(x,\Pb(x,u),V)$, we consider the random mapping $\sigma^{\text{b}}(x,\Pb^{(N)}(x,u),V)$. Its expected value is a new coherent transition risk mapping
\begin{equation}
\label{mini-batch}
\sigma^{(N)}(x,\Pb(x,u),V)=\Eb_{Y^{1:N}\sim \Pb(x,u)^N}\big[\sigma^{\text{b}}(x,\Pb^{(N)}(x,u),V)\big].
\end{equation}
An essential virtue of this construction is that we can write $\sigma^{\text{b}}(x,\Pb^{(N)}(x,u),V)$ as a function of $N$ values: $V(Y^j)$, $j=1,\dots,N$, at the batch $Y^{1:N}$.

The verification that $\sigma^{(N)}(x,\Pb(x,u), \Cdot)$ is coherent is straightforward, directly from the four axioms. Evidently, if $\sigma^{\text{b}}(x,\Pb(x,u),\Cdot)$
is the conditional expectation \eqref{E-form}, the mini-batch mapping \eqref{mini-batch} is the same, due to the tower property. However, in the examples
\eqref{msd-form} and \eqref{sup-form}, the formula \eqref{mini-batch} defines  new transition risk mappings. Thanks to the small batch involved and the use of the
expectation, their statistical estimation and learning become tractable, as we shall demonstrate below.
Particularly convenient is the mini-batch counterpart of the worst-case mapping \eqref{sup-form}:
\begin{equation}
\label{batch-sup}
\sigma^{(N)}(x,\Pb(x,u),V) = \Eb_{Y^{1:N}\sim \Pb(x,u)^N}\big[ \textstyle{\max_{1 \le j \le N}} V(Y^j) \big].
\end{equation}
Even for $N=2$ it defines a nontrivial yet tractable model of risk aversion. We do not consider $N\to \infty$, but rather work with a fixed and very small $N$.

The following proposition identifies the mapping \eqref{batch-sup} within the
classical taxonomy of risk measures. It shows that the batch size $N$ is not
merely an algorithmic parameter, but indexes a one-parameter family of
distortion measures of risk, with a Kusuoka representation available in closed
form.

\begin{proposition}
\label{prop:structure}
Fix $x\in\Xc$ and $u\in \Uc(x)$, write $\Pb=\Pb(x,u)$, and regard $V \in \Vc$ as a
random variable on the probability space $(\Xc,2^{\Xc},\Pb)$, with the
distribution function $F_V$ and the quantile function $F_V^{-1}$. Then the
mapping \eqref{batch-sup} has the following properties.
\begin{tightlist}{iii}
\item[\textup{(i)}] It is a Choquet integral with respect to the distorted
probability $g_N\circ \Pb$, with the distortion function $g_N(s)=1-(1-s)^N$:
\[
\sigma^{(N)}(x,\Pb,V)=\int_0^{\infty} g_N\big(\Pb\{V>\tau\}\big)\,\D \tau
 + \int_{-\infty}^{0}\Big[ g_N\big(\Pb\{V>\tau\}\big)-1\Big]\,\D \tau .
\]
As $g_N$ is increasing and concave, with $g_N(0)=0$ and $g_N(1)=1$, the mapping
$\sigma^{(N)}(x,\Pb,\Cdot)$ is a coherent, law invariant, and comonotone
additive measure of risk.
\item[\textup{(ii)}] It is a spectral measure of risk with the nondecreasing
spectrum $\phi_N(s)=Ns^{N-1}$, and, for $N\ge 2$, its Kusuoka representation is
the Beta mixture of Average Values-at-Risk:
\[
\sigma^{(N)}(x,\Pb,V)=\int_0^1 N s^{N-1} F_V^{-1}(s)\,\D s
=\int_0^1 \avar_{\beta}(V)\,\mu_N(\D \beta),
\]
where
\[
\mu_N(\D \beta)= N(N-1)\beta^{N-2}(1-\beta)\,\D \beta,
\]
that is, $\mu_N$ is the \textup{Beta}$(N-1,2)$ distribution.
\item[\textup{(iii)}] For $N=2$,
\[
\sigma^{(2)}(x,\Pb,V)= \Eb[V] + \tfrac{1}{2}\,\Eb\big|V(Y^1)-V(Y^2)\big|
= \Eb[V] + \Gini(V),
\]
where $\Gini(\Cdot)$ is one half of Gini's mean difference. Consequently, the
mixed mapping appearing in \eqref{Q-risk} with $N=2$ is the mean--Gini
model $\Eb[V]+\varkappa\,\Gini(V)$.
\item[\textup{(iv)}] The family is nondecreasing in the batch size and exhausts
the worst case:
\[
\Pb V = \sigma^{(1)}(x,\Pb,V) \le \sigma^{(N)}(x,\Pb,V)\le \sigma^{(N+1)}(x,\Pb,V)
\le \esssup_{\Pb} V,
\]
and $\sigma^{(N)}(x,\Pb,V)\uparrow \esssup_{\Pb} V$ as $N\to\infty$.
\item[\textup{(v)}] Every member of the family is consistent with second order
stochastic dominance: if $V\icx W$, where $\icx$ denotes the increasing convex
order, which is the dominance relation appropriate for costs, then
$\sigma^{(N)}(x,\Pb,V)\le\sigma^{(N)}(x,\Pb,W)$.
\end{tightlist}
\end{proposition}

\begin{proof}
(i) For every $\tau\in\Rb$, by the independence of $Y^1,\dots,Y^N$,
\[
\Pb\Big\{ \max_{1\le j\le N} V(Y^j) > \tau \Big\}
= 1 - \big( 1 - \Pb\{V>\tau\}\big)^N = g_N\big(\Pb\{V>\tau\}\big),
\]
and the displayed formula is the representation of the expected
value of $\max_{1\le j\le N} V(Y^j)$. Moreover, $g_N'(s)=N(1-s)^{N-1}> 0$ and
$g_N''(s)=-N(N-1)(1-s)^{N-2}\le 0$ on $[0,1)$. Coherence, law invariance, and
comonotone additivity of Choquet integrals with concave distortion functions
are classical; see \cite[Ch.~2]{dentcheva2024risk} and \cite{follmer2006convex}.

(ii) We use the spectral representation of law invariant comonotone additive
coherent measures of risk \cite{kusuoka2001law,acerbi2002spectral}.
The distribution function of $\max_{1\le j\le N}V(Y^j)$ is $F_V^N$, and its
quantile function is $s\mapsto F_V^{-1}(s^{1/N})$. Hence
$\sigma^{(N)}(x,\Pb,V)=\int_0^1 F_V^{-1}(s^{1/N})\,\D s$, and the substitution
$s = v^N$ yields the spectral form. Since $\avar_\beta$ is the spectral measure
with the spectrum $u\mapsto \1_{[\beta,1]}(s)/(1-\beta)$, a mixture
$\int_0^1 \avar_\beta(\Cdot)\,\mu(\D\beta)$ has the spectrum
$\phi(s)=\int_{[0,s]}(1-\beta)^{-1}\mu(\D\beta)$; therefore
$\mu(\D\beta)=(1-\beta)\phi'(\beta)\,\D\beta$, with no atom at $0$ because
$\phi_N(0)=0$ for $N\ge 2$. Substituting $\phi_N'(\beta)=N(N-1)\beta^{N-2}$
gives $\mu_N$, whose total mass is $N(N-1)B(N-1,2)=1$.

(iii) Immediate from $\max(a,b)=\frac{1}{2}(a+b)+\frac{1}{2}|a-b|$ and the fact
that $Y^1$ and $Y^2$ are independent copies.

(iv) We have $g_N\le g_{N+1}$ pointwise on $[0,1]$, and the Choquet integral is
monotone with respect to the distortion function; the limit follows from
$\max_{1\le j\le N}V(Y^j)\uparrow \esssup_\Pb V$ almost surely and monotone
convergence.

(v) For $N>1$, by (ii), $\sigma^{(N)}$ is a mixture of the mappings $\avar_\beta$, each of
which is consistent with the increasing convex order
\cite{OgryczakRuszczynski2002,dentcheva2024risk}.
\end{proof}

Part (iii) deserves emphasis. Combining \eqref{batch-sup} with the conditional
expectation with the weight $\varkappa\in[0,1]$, as we do in \eqref{Q-risk}
below, produces for $N=2$ the mean--Gini model, one of the oldest mean--risk
models known to be consistent with second order dominance
\cite{yitzhaki1982stochastic,OgryczakRuszczynski2002}, and the admissible
range $\varkappa \in [0,1]$ is exactly the range in which that consistency was
established. The choice $N=2$, which we make in all our experiments, is
therefore not a computational compromise: it delivers a classical and
well-understood model of risk at the cost of two next-state samples per
transition. 

We also obtain the Q-factors:
\begin{equation}
\label{Q-risk}
\begin{aligned}
Q^\pi(x,u) &= c(x,u) + \alpha\,\Eb_{Y^{1:N}\sim \Pb(x,u)^N} \Big[ (1-\varkappa) \frac{1}{N} \sum_{j=1}^N V^\pi(Y^j)
+ \varkappa {\max_{1 \le j \le N}} V^\pi(Y^j)\Big],\\
Q^\star(x,u) &= c(x,u) + \alpha\,\Eb_{Y^{1:N}\sim \Pb(x,u)^N} \Big[ (1-\varkappa) \frac{1}{N} \sum_{j=1}^N V^\star(Y^j)
+ \varkappa {\max_{1 \le j \le N}} V^\star(Y^j)\Big].
\end{aligned}
\end{equation}
Being expectations, they admit simple unbiased estimators:
\begin{align*}
\widetilde{Q}^\pi(x,u) &= c(x,u) +  \alpha\Big[(1-\varkappa) \frac{1}{N} \sum_{j=1}^N V^\pi(Y^j)
+ \varkappa {\max_{1 \le j \le N}} V^\pi(Y^j)\Big],\\
\widetilde{Q}^\star(x,u) &= c(x,u) +  \alpha\Big[(1-\varkappa) \frac{1}{N} \sum_{j=1}^N V^\star(Y^j)
+ \varkappa {\max_{1 \le j \le N}} V^\star(Y^j)\Big].
\end{align*}
Here and below, $\varkappa$ is the weight of the worst-case component. Both estimators are unbiased, conditionally on $(x,u)$,
and each of them uses one single batch of $N$ next states. This is the decisive
computational advantage of the construction \eqref{mini-batch}: no
double sampling and no pointwise estimation of a nonlinear functional of the
transition kernel are needed.

\section{Risk-Averse Q-Learning with Double Deep Q-Networks}
\label{sec:ddqn}

As in the expected-value case, the direct application of the Q-factors \eqref{Q-risk} is limited to small state and action spaces. In our application to be discussed in the next sections the tabular approach is out of the question, and
we resort to functional approximations within a pre-specified parametric function class. 

Linear value function approximation approaches to MDPs have a long history; see
\cite{bradtke1996linear,sutton1988learning,tsitsiklis1997analysis,Sutton1998,melo2007q}
and many references therein. The concept of a linear (mixture) MDP \cite{bradtke1996linear,melo2007q}, originating from the theory of linear bandits \cite{bubeck2012regret,lattimore2020bandit},  is pertinent in this setting.
It has been recently used by
\cite{ayoub2020model,yang2020reinforcement,jin2020provably,modi2020sample} to develop complexity bounds for expected value RL algorithms.
The generalization to bilinear models by \cite{du2021bilinear} extends the applicability of this class. The state aggregation or lumping approaches of \cite{dong2019provably,katehakis2012successive} are also related to this model.

In \cite{ruszczynski2025risk}, we used a linear functional model to approximate the mini-batch policy value function \eqref{DP-risk-infinite} within a policy iteration method, and we proved statistical error bounds of
these approximations. Now, we intend to adapt the techniques of deep neural networks within a Q-learning scheme.

We approximate the optimal Q-factors $Q^\star(x,u)$ by a neural
network $Q_\theta:\Xc\times\Uc\to\Rb$ with parameters $\theta\in\Rb^d$, and
introduce a second \emph{target network} $Q_{\theta^-}$ with parameters
$\theta^-$, periodically synchronized with $\theta$.  The rationale for
maintaining two networks, the \emph{Double Deep Q-Network} (DDQN) architecture of
\citet{van2016deep}, stems from the systematic bias inherent in single-network
Q-learning. Using the same parameter vector for both action selection and action
evaluation, and assuming that $Q_\theta(y,u')$ is an unbiased estimator of
$Q^\star(y,u')$ for each fixed $u'$, Jensen's inequality applied to the concave
function $\min$ yields
\[
    \Eb \bigl[\min_{u'\in\Uc(y)} Q_\theta(y,u')\bigr]
    \;\le\;
    \min_{u'\in\Uc(y)} \Eb\bigl[ Q_\theta(y,u')\bigr]
    \;=\;
    \min_{u'\in\Uc(y)} Q^\star(y,u').
\]
In the maximization convention of \citet{van2016deep} this is the familiar
overestimation bias; in our minimization convention it appears as a systematic
\emph{underestimation} of the optimal cost-to-go, that is, as unwarranted
optimism. Its effect is unfavorable in either convention: the target
is biased towards the appearance of a better policy than the one actually
available.

In the risk-averse setting a second and distinct source of bias is present. The
sample estimator $\widetilde{Q}^\star$ of Section~\ref{sec:mini-batch} involves
two optimizations: (i) action selection via $V^\star(y)=\min_{u'}Q^\star(y,u')$,
and (ii) the worst-case term $\max_{1\le j\le N}V^\star(Y^j)$ of the mini-batch
risk mapping~\eqref{batch-sup}.  The DDQN prescription addresses (i) directly,
as we now describe; source~(ii) is of a different nature and is analyzed in
Remark~\ref{rem:state-max-bias} below.

Recall that the optimal value function satisfies
$V^\star(y)=\min_{u'\in\Uc(y)}Q^\star(y,u')$.  The DDQN estimator of this
quantity decouples selection from evaluation:
\begin{equation}
    \label{V-ddqn}
    V^{\mathrm{dq}}(y;\theta,\theta^-)
    \;=\;
    Q_{\theta^-}\!\Bigl(y,\;\argmin_{u'\in\Uc(y)} Q_\theta(y,u')\Bigr).
\end{equation}
The \emph{online} network $\theta$ identifies the greedy action; the
\emph{target} network $\theta^-$ evaluates its worth.  This separation removes
the correlation between the selected action and its estimated value
that drives overestimation.

Given a mini-batch of next states $Y^{1:N}{\sim}\Pb(x,u)^N$,
we substitute $V^{\mathrm{dq}}$ for $V^\star$ in the sample estimator
$\widetilde{Q}^\star$, obtaining the risk-averse DDQN Bellman target
\begin{equation}
    \label{ddqn-target}
    \Tc(x,u,Y^{1:N};\theta,\theta^-)
    \;=\;
    c(x,u)
    + \alpha\biggl[
        (1-\varkappa)\frac{1}{N}\sum_{j=1}^{N} V^{\mathrm{dq}}(Y^j;\theta,\theta^-)
        + \varkappa\max_{1\le j\le N} V^{\mathrm{dq}}(Y^j;\theta,\theta^-)
    \biggr].
\end{equation}
Here $\alpha\in(0,1]$ is the discount factor.  Taking expectations over
$Y^{1:N}$ in~\eqref{ddqn-target} recovers the mini-batch Q-factor associated with the
mapping~\eqref{mini-batch}, confirming consistency of the target with the
coherent risk structure established in Section~\ref{sec:mini-batch}.

\begin{remark}
    \label{rem:state-max-bias}
    {\rm 
    The term $\max_{1\le j\le N}V^{\mathrm{dq}}(Y^j)$ maximizes over realized
    \emph{next states}, and not over competing action-value estimates.
    Conditionally on $(\theta,\theta^-)$, the function
    $y\mapsto V^{\mathrm{dq}}(y;\theta,\theta^-)$ is deterministic, and
    therefore
    \[
        \Eb_{Y^{1:N}\sim\Pb(x,u)^N}\Big[\max_{1\le j\le N}
        V^{\mathrm{dq}}(Y^j;\theta,\theta^-)\Big]
        \;=\;
        \sigma^{(N)}\big(x,\Pb(x,u),V^{\mathrm{dq}}\big)
    \]
    holds \emph{exactly}. The batch maximum is thus an unbiased estimator of the
    mini-batch risk mapping evaluated at the current approximation. This property, established in Section~\ref{sec:mini-batch}, 
    makes the whole scheme implementable.

    What remains is the propagation of the approximation error
    $V^{\mathrm{dq}}-V^\star$ through $\sigma^{(N)}$. By \eqref{nonexpansive},
    this error is not amplified,
    \[
        \big|\sigma^{(N)}(x,\Pb,V^{\mathrm{dq}})-\sigma^{(N)}(x,\Pb,V^\star)\big|
        \le \|V^{\mathrm{dq}}-V^\star\|_\infty .
    \]
    Its expectation with respect to the randomness of training, denoted
    $\Eb_\theta$, does not vanish, however: since $\sigma^{(N)}(x,\Pb,\Cdot)$ is
    convex, Jensen's inequality gives
    \[
        \Eb_\theta\big[\sigma^{(N)}(x,\Pb,V^{\mathrm{dq}})\big]
        \;\ge\;
        \sigma^{(N)}\big(x,\Pb,\Eb_\theta[V^{\mathrm{dq}}]\big),
    \]
    so that zero-mean noise in the value estimates contributes a nonnegative,
    that is, \emph{pessimistic} bias proportional to $\varkappa$. It is worth stressing that this bias
    has the sign \emph{opposite} to the optimism of the single-network target
    discussed above. 
    }
\end{remark}

Because the target~\eqref{ddqn-target} requires $N$ next-state samples
per state-action pair, experience tuples stored in the replay buffer
$\Dc$ take the form
\[
    e = \bigl(x,\,u,\,c,\,Y^{1:N}\bigr),
    \quad
    Y^{1:N}\sim\Pb(x,u)^N.
\]
In simulation-based settings the $N$ samples are drawn directly from the
transition kernel. 

Given a training mini-batch $\Bc\subset\Dc$ of $K$ experience tuples, the
empirical squared-error loss is
\begin{equation}
    \label{ddqn-loss}
    \Lc(\theta;\theta^-)
    \;=\;
    \frac{1}{K}
    \sum_{(x,u,c,Y^{1:N})\in\Bc}
    \Bigl[
        \Tc(x,u,Y^{1:N};\theta,\theta^-) - Q_\theta(x,u)
    \Bigr]^2.
\end{equation}
The gradient $\nabla_\theta\Lc$ is computed treating $\theta^-$ as a constant,
so the Bellman target~\eqref{ddqn-target} does not propagate gradients through
the target network.  This ``stop-gradient'' convention, standard since
\cite{mnih2015human}, ensures the regression target remains stationary across
each gradient step, stabilizing training.

\begin{algorithm}[tb]
\caption{Risk-Averse DDQN with Mini-Batch Transition Risk}
\label{alg:ra-ddqn}
\KwIn{Mini-batch size $N$, replay buffer capacity, training batch size $K$,
      target update period $\tau$, risk weight $\varkappa\in[0,1]$,
      discount $\alpha$, learning rate $\beta$, exploration schedule $\{\varepsilon_t\}$}
Initialize online network $Q_\theta$; set $\theta^-\leftarrow\theta$\;
Initialize replay buffer $\Dc\leftarrow\emptyset$\;
\For{each episode}{
    Observe initial state $x_0$\;
    \For{each step $t=0,1,2,\dots$}{
        With prob.\ $\varepsilon_t$ select $u_t$ at random;
        otherwise $u_t\leftarrow\argmin_{u}Q_\theta(x_t,u)$\;
        Execute $u_t$; observe cost $c_t$;
        draw $Y_t^{1:N}\sim\Pb(x_t,u_t)^N$\;
        Store $(x_t,u_t,c_t,Y_t^{1:N})$ in $\Dc$\;
        Sample training batch $\Bc=\{(x_i,u_i,c_i,Y_i^{1:N})\}_{i=1}^K$ from $\Dc$\;
        \For{$i=1,\dots,K$}{
            $u_{ij}^\star \leftarrow \argmin_{u'\in\Uc(Y_i^j)}Q_\theta(Y_i^j,u')$,
            \quad $j=1,\dots,N$\;
            $V_{ij}^{\mathrm{dq}} \leftarrow Q_{\theta^-}(Y_i^j,u_{ij}^\star)$,
            \quad $j=1,\dots,N$\;
            $\Tc_i \leftarrow c_i + \alpha\bigl[
                (1-\varkappa)\tfrac{1}{N}\textstyle\sum_{j} V_{ij}^{\mathrm{dq}}
                +\varkappa\max_{j} V_{ij}^{\mathrm{dq}}\bigr]$\;
        }
        $\Lc(\theta;\theta^-) \leftarrow \frac{1}{K}\sum_{i=1}^{K}
            \bigl[\Tc_i - Q_\theta(x_i,u_i)\bigr]^2$
            \quad ($\Tc_i$ treated as a constant)\;
        $\theta \leftarrow \theta
            - \beta\,\nabla_\theta\Lc(\theta;\theta^-)$\;
        \If{$t \bmod \tau = 0$}{
            $\theta^-\leftarrow\theta$ \quad(hard update)
            \textbf{ or }\quad
            $\theta^-\leftarrow(1-\rho)\,\theta^-+\rho\,\theta$
            \quad(soft update)\;
        }
    }
}
\end{algorithm}

Setting $\varkappa=0$ reduces~\eqref{ddqn-target} to the standard DDQN
target with $N$ averaged next states, which collapses to the classical
DDQN of \citet{van2016deep} when $N=1$.  Setting $\varkappa=1$ and letting
$N\to\infty$ recovers a purely worst-case, that is, robust target, by
Proposition~\ref{prop:structure}(iv).  The pair $(\varkappa,N)$ thus provides
two complementary dials: $\varkappa$ interpolates linearly between risk-neutral
expected-cost minimization and the extended Gini mean of order $N$, while $N$
controls how far that mean is displaced towards the upper tail. Both preserve
the coherence and the dominance consistency established in
Section~\ref{sec:mini-batch}, and the corresponding Bellman operator remains an
$\alpha$-contraction by \eqref{nonexpansive}.

Algorithm \ref{alg:ra-ddqn} presents the outline of our idea.

\section{Risk-Aware Underwater Robot Navigation}
\label{s:robot}

In this and the following sections we consider  the underwater robot navigation problem discussed
in \cite{ruszczynski2025risk}. We formulate it
as a Markov decision
process and apply the mini-batch transition risk framework developed in
the preceding sections to obtain a risk-averse control policy.

\subsection{Problem Overview}

The robot operates in a hazardous underwater environment. Its task is to visit a
set of \textit{collection points}, gather stochastic information payloads, and
transmit them at designated \textit{transmission points}, while avoiding known
obstacle zones. The robot is subject to \textit{systematic risk}: at each time
step it is destroyed with probability $\delta\in(0,1)$, independently of its
action. The objective is to minimize the cumulative risk of information loss
while completing the mission efficiently.
Throughout, $\delta$ denotes this destruction probability and $\alpha$ the
discount factor of Sections~\ref{s:Markov_risk}--\ref{sec:ddqn}.


The environment is modeled as a closed, connected rectangular grid
$\Gc\subset\Zb^2$. The following disjoint subsets of $\Gc$ are fixed
throughout an episode:
\begin{itemize}
    \item \textit{Collection points} $\mathcal{C}\subset\Gc$: locations where
    information is gathered.
    \item \textit{Transmission points} $\Rc\subset\Gc$: locations where
    carried information is offloaded.
    \item \textit{Obstacles} $\mathcal{O}\subset\Gc$: locations the robot must
    not enter; $\mathcal{C}$, $\Rc$, and $\mathcal{O}$ are pairwise
    disjoint.
\end{itemize}
They constitute the \textit{configuration} $\Ec$ of the problem.

At each collection point, a visit yields high-value information $I_H$ with
probability $1-p$ and low-value information $I_L < I_H$ with probability $p$.
Formally, the collected payload is the random variable
\[
    I_{\text{obs}} =
    \begin{cases}
        I_L, & \text{with probability } p,\\
        I_H, & \text{with probability } 1-p.
    \end{cases}
\]
The set of attainable accumulated information payloads is $\mathcal{I}\subset\Rb_+$.

\paragraph{Action space.}
The robot's action space $\mathcal{U} = \mathcal{M}\cup\{a_c, a_t\}$ consists
of three components:
\begin{itemize}
    \item \textbf{Movement} ($\mathcal{M}$): The set
    $\mathcal{M} = \{-1,0,1\}^2\setminus\{(0,0)\}$ comprises eight
    directions (N, NE, E, SE, S, SW, W, NW). Movement to an obstacle cell or
    outside $\Gc$ is inadmissible.

    \item \textbf{Collect} ($a_c$): Available when the robot occupies a
    collection point $g\in\mathcal{C}$. Executing $a_c$ draws
    $I_{\text{obs}}$ from the Bernoulli mixture above and adds it to the
    robot's carried payload $I_{\text{carry}}$.

    \item \textbf{Transmit} ($a_t$): Available when the robot occupies a
    transmission point $g\in\Rc$ and $I_{\text{carry}}>0$.
    Executing $a_t$ delivers the payload and resets $I_{\text{carry}}$ to
    zero.
\end{itemize}


A state $x\in\mathcal{X}$ is a tuple
\[
    x = \bigl(g_r,\, \mathbf{v},\, I_{\text{carry}}\bigr),
\]
where $g_r\in\Gc$ is the robot's current grid position,
$\mathbf{v}\in\{0,1\}^{|\mathcal{C}|}$ is a binary vector whose $i$-th entry
equals $1$ if collection point $c_i$ has not yet been visited, and
$I_{\text{carry}}\in\mathcal{I}$ is the payload currently held. The configuration 
$\Ec = (\mathcal{C},\Rc,\mathcal{O})$ is a fixed instance
parameter, not a dynamic state. The state space is therefore
\[
    \mathcal{X} = \Gc \times \{0,1\}^{|\mathcal{C}|} \times \mathcal{I}.
\]
The episode terminates when the robot is destroyed or it enters the
\textit{terminal state space}
\[
    \mathcal{X}_T = \Gc \times \{0\}^{|\mathcal{C}|} \times \{0\},
\]
which encodes the condition that all collection points have been visited and no
payload remains to be transmitted.

\subsection{Cost Structure}
\label{s:costs}

The one-step cost $c(x,u)$ is the sum of the component costs defined below.
Since the collection cost depends on the stochastic payload $I_{\text{obs}}$,
the transition risk mapping $\sigma^{(N)}$ of Section~\ref{sec:mini-batch} acts
on the next-state value function and implicitly aggregates this randomness.

\begin{itemize}

    \item \textbf{Collection cost.}
    Executing $a_c$ at a collection point incurs a cost proportional to the
    realized payload:
    \[
        C_{\text{collect}} = c_o + k_o\,I_{\text{obs}},
    \]
    where $c_o\ge 0$ is a fixed observation overhead and $k_o>0$ is the
    per-unit information loss rate. Since $I_{\text{obs}}$ is random,
    $C_{\text{collect}}$ is a random variable with expectation
    $c_o + k_o\bigl[p\,I_L + (1-p)\,I_H\bigr]$.

    \item \textbf{Movement cost.}
    Each movement step incurs a cost proportional to the carried payload, with
    two regimes. In the \emph{survival regime} used during training, the cost is
    weighted by the probability of surviving the step:
    \[
        C_{\text{move}}^{\text{surv}} = (1-\delta)\,\bigl(c_m + k_m\,I_{\text{carry}}\bigr),
    \]
    where $c_m\ge 0$ is a base transit cost and $k_m>0$ is the per-unit
    information loss rate. In the \emph{destruction regime} used for
    scenario-based evaluation, the full cost is incurred unconditionally,
    reflecting total loss upon destruction:
    \[
        C_{\text{move}}^{\text{destr}} = c_m + k_m\,I_{\text{carry}}.
    \]

    \item \textbf{Transmission cost.}
    Executing $a_t$ at a transmission point yields a negative cost (reward)
    equal to the delivered payload. Arriving at a transmission point without
    a payload incurs a penalty $C_{\text{none}}\ge 0$:
    \[
        C_{\text{transmit}} =
        \begin{cases}
            -I_{\text{carry}}, & \text{if } I_{\text{carry}} > 0,\\
            C_{\text{none}},   & \text{otherwise.}
        \end{cases}
    \]

\end{itemize}

\begin{remark}
\label{rem:killing}
{\rm
Destruction is modeled in the standard way, by adjoining an absorbing state
$\mathfrak{d}$ with $V(\mathfrak{d})=0$ and letting the kernel place the mass
$\delta$ on $\mathfrak{d}$ at every step. If $\sigma$ is the conditional
expectation \eqref{E-form}, this device is \emph{exactly} equivalent to
evaluating the surviving trajectory with the discount factor
$\alpha = 1-\delta$; this is the classical equivalence between killing and
discounting \cite{Puterman1994}. For a nonlinear transition risk mapping the
equivalence fails.  Absorbing the
destruction probability into a scalar discount factor is therefore an
approximation whose error grows with $\varkappa$ and $N$. For the value $N=2$
used throughout our experiments the discrepancy is modest, but it is not zero,
and it is one of the reasons why in Section~\ref{s:results} we evaluate the
learned policies under \emph{explicitly simulated} destruction events, rather
than only under the model on which they were trained.
}
\end{remark}

\subsection{Decomposed Training Methodology}
\label{s:training}

A monolithic DQN operating on the full action space $\mathcal{U}$ over $\Xc$
failed to learn a viable policy: joint exploration of spatial and task decisions within
various configurations
creates a  problem that standard experience replay does not
resolve within feasible training budgets. 

We therefore adopt a
\textit{decomposed} architecture separating high-level task decisions from
low-level path execution.

\paragraph{High-level decisions.}
At each decision epoch the agent selects a \emph{target}---either the nearest
unvisited collection point or the nearest transmission point---using a
feature-based risk-averse DDQN as described in Section~\ref{sec:ddqn}.

\paragraph{Low-level path execution.}
Given the target, the sequence of movement actions navigating the robot from
its current position to the target---while avoiding obstacles and respecting
grid boundaries---is computed by a classical shortest-path algorithm. This
sub-problem admits a known deterministic solution; delegating it to a classical
planner eliminates function-approximation error and concentrates the learning
budget on the genuinely uncertain high-level decisions.


The low-level planning sub-problem is solved by the Dijkstra
algorithm that finds a collision-free direction sequence from the current cell
to the target cell. It is preferred over learning-based alternatives because
obstacle avoidance on a finite, fully observed grid is a deterministic problem
for which graph search is exact, interpretable, and computationally negligible
relative to the learning component. Full details appear in  section \ref{s:Algorithm}.


The high-level decision is posed over the binary action space
\[
\mathcal{U}_H = \big\{\,\text{target the nearest } c\in\mathcal{C},\;\;
\text{target the nearest } t\in\Rc\,\big\},
\]
and the optimal Q-factors over $\mathcal{U}_H$ are approximated by the
risk-averse DDQN of Section~\ref{sec:ddqn} with the mini-batch
target~\eqref{ddqn-target}.

\paragraph{Heuristic baseline policy.}
As a benchmark we consider the threshold policy $\varPi(\gamma\,;\,\cdot)$ introduced in 
\cite{ruszczynski2025risk},
\[
\varPi(\gamma;x) =
\begin{cases}
    \text{move to the nearest transmission point,} &
    \hspace{-2em}\text{if } I_{\text{carry}} > 0 \text{ and }
    d_{\mathcal{C}} \ge \gamma(\mathcal{C},\Rc)\,
    \dfrac{d_{\Rc}}{I_{\text{carry}}},\\[6pt]
    \text{move to the nearest unvisited collection point,} & \text{otherwise,}
\end{cases}
\]
where $d_{\mathcal{C}}$ and $d_{\Rc}$ are the shortest distances from
the current position to the nearest unvisited collection point and to the
nearest transmission point, respectively, and $\gamma(\mathcal{C},\Rc)>0$
is an instance-dependent threshold. The policy is applicable to any connected
state space equipped with a suitable distance function.

Calibrating $\gamma$ requires policy iteration with line search, which is
computationally expensive and yields a parameter that is instance-specific
and non-transferable. We refer to \cite{ruszczynski2014erratum} for details. The heuristic policy  serves as a performance
baseline against which the learned DDQN policy is compared.

\paragraph{Feature representations.}
To assess the sensitivity of the DDQN approximation to the input encoding,
we evaluate two feature representations derived from the state $x$ and the configuration $\Ec$:

\textbf{1) Raw state and configuration encoding.}
The instance and state data are flattened into a single input vector comprising one-hot encoding
of the configuration $\Ec$, concatenated
with the state $x$.

\textbf{2) Engineered features.}
A compact feature vector is constructed from the following quantities,
identified as most informative for the high-level decision through analysis and
systematic experimentation:
\begin{tightitemize}
   \item Fraction of unvisited collection points, $\|\mathbf{v}\|_1/\|\mathcal{C}\|_1$;
    \item Mean and standard deviation of pairwise distances among unvisited
          collection points;
    \item Grid distance $d_{\mathcal{C}}$ and radial distance $r_{\mathcal{C}}$  to the nearest unvisited collection point;
    \item Grid distance $d_{\Rc}$ and radial distance $r_{\Rc}$ to the nearest transmission point;
    
    \item Carried payload $I_{\text{carry}}$;
    \item Mean and standard deviation of the distance from each unvisited
          collection point to its nearest transmission point.
\end{tightitemize}
The ``grid distance'' is understood as the length of the shortest path, with accounting of all obstacles. 

One of the important motives for selecting the engineered features was the
observation that the problem, viewed abstractly on the unbounded lattice
$\Zb^2$, possesses an obvious symmetry group $G$ defined by the lattice translations, rotations, and reflections.
 Consequently, the optimal value function
is invariant with respect to these operations,  and the optimal policy -- equivariant.

The raw encoding does not reflect this structure: two geometrically identical states 
are mapped to different input vectors, and the network has to learn the
invariance from the data, which inflates the sample requirement by a factor of
order $|G|$. Every quantity in the list above, by contrast, is a $G$-invariant
of the pair (state, configuration): cardinalities, distances, and the moments
of distance distributions are unchanged by isometries. The engineered
representation therefore factors the learning problem through the quotient
$\Xc/G$. Moreover, because the high-level action set $\Uc_H$ is defined
intrinsically, by \emph{the nearest unvisited collection point} and \emph{the nearest
transmission point}, rather than by grid directions, the induced policy is
automatically equivariant, and none of the features depends on the size of the
grid or on the number of collection points. The latter property is what makes
the transfer to instance sizes never seen in training, reported in
Section~\ref{s:results}, possible at all.

Results for both encodings are reported in Section~\ref{s:results};
algorithmic details of the full training procedure appear in the Algorithm
section.\vspace{1ex}

\begin{remark}
\label{rem:smdp}
{\rm 
Selecting a target rather than a primitive move makes the high-level process a
semi-Markov decision process, or, in the terminology of
\citet{sutton1999between}, a decision process over \emph{options}. A
macro-action initiated at $x$ terminates after a random number $k(x,u)$ of
primitive steps, and in the expected-value case the correct Bellman target
discounts by $\alpha^{k(x,u)}$ and accumulates the intermediate costs as
$\sum_{i=0}^{k-1}\alpha^{i}c_i$. In the risk-averse case the situation is more
delicate: time consistency requires the option-level mapping to be the
\emph{composition} of the $k$ one-step mappings, and, unless $\sigma$ is the
conditional expectation, the composition of $k$ copies of $\sigma$ is not
$\sigma$ applied to the accumulated cost. Two features of our model keep the
discrepancy small. The low-level path is deterministic given the target, so
that the only randomness inside a macro-action is the destruction event and the
payload draw at the terminal cell of the path; and the movement costs are
accumulated with the survival weighting of Section~\ref{s:costs}.
}

\end{remark}

\subsection{Algorithm}
\label{s:Algorithm}

The reinforcement learning framework consists of two major components:
(i) a path-finding module that computes the optimal navigation and control matrices in a grid-based environment (Algorithm \ref{alg:pathfinding}), and 
(ii) Hierarchical Double Deep Q-Network (Algorithm \ref{alg:hra-ddqn}).

\begin{algorithm}[H]
\caption{Path-Finding and Control Computation}
\label{alg:pathfinding}
\KwIn{Start point $g_0$, grid $\mathcal{G}$, obstacle set $\mathcal{O}$, move set $\mathcal{M}$, norm order $r$, maximum distance $d_{\max}$}
\KwOut{Distance matrix $D$, action (control) matrix $A$}

Initialize $D(x,y) \gets d_{\max}$ and $A(x,y) \gets \varnothing$ for all $(x,y) \in \mathcal{G}$\;
Set $D(g_0) \gets 0$\;
Initialize candidate set $Q \gets \mathcal{G} \setminus \mathcal{O}$\;

\While{$Q$ not empty}{
    Select $(x,y) = \arg\min_{(i,j) \in Q} D(i,j)$\;
    \If{$D(x,y) \ge d_{\max}$}{
        \textbf{terminate} (graph not connected)\;
    }
    Remove $(x,y)$ from $Q$\;

    \ForEach{$m \in \mathcal{M}$}{
        Compute $(x',y') = (x,y) + m$\;
        \If{$(x',y') \in Q$}{
            $\tilde{d} \gets D(x,y) + \lVert (x',y') - (x,y) \rVert_r$\;
            \If{$\tilde{d} < D(x',y')$}{
                $D(x',y') \gets \tilde{d}$\;
                $A(x',y') \gets m$\;
            }
        }
    }
}
\Return{$D$, $A$}
\end{algorithm}

\begin{algorithm}[t]
\caption{Hierarchical Risk-Averse DDQN with Mini-Batch Transition Risk}
\label{alg:hra-ddqn}
\KwIn{Mini-batch size $N$, replay buffer capacity, training batch size $K$,
      target update period $\tau$, risk weight $\varkappa\in[0,1]$, number of episodes $J_{\text{train}}$,
      discount $\alpha$, learning rate $\beta$, exploration schedule $\{\varepsilon_t\}$}
Initialize online network $Q_\theta$; set $\theta^-\leftarrow\theta$\;
Initialize replay buffer $\Dc\leftarrow\emptyset$\;

\For{$i=1,\dots,J_{\textup{train}}$}{
    Observe initial state $x_0$\;

    \For{each step $t=0,1,2,\dots$}{

        \textbf{Macro-action selection:} \\
        With prob.\ $\varepsilon_t$ select $u_t^{\mathrm{macro}}$ at random; \\
        otherwise $u_t^{\mathrm{macro}} \leftarrow \argmin_{u} Q_\theta(x_t,u)$\;

        \textbf{Micro-action execution:} \\
        Determine primitive action $u_t$ via shortest path policy (Algorithm \ref{alg:pathfinding}\ ) according to $u_t^{\mathrm{macro}}$ \;

        Execute $u_t$; observe cost $c_t$;
        draw $Y_t^{1:N}\sim\Pb(x_t,u_t)^N$\;

        Store $(x_t,u_t^{\mathrm{macro}},c_t,Y_t^{1:N})$ in $\Dc$\;

        Sample training batch $\Bc=\{(x_k,u_k^{\mathrm{macro}},c_k,Y_k^{1:N})\}_{k=1}^{K}$ from $\Dc$\;

        \For{$k=1,\dots,K$}{

            $u_{kj}^\star \leftarrow \argmin_{u'\in\Uc(Y_k^j)}Q_\theta(Y_k^j,u')$,
            \quad $j=1,\dots,N$\;

            $V_{kj}^{\mathrm{dq}} \leftarrow Q_{\theta^-}(Y_k^j,u_{kj}^\star)$,
            \quad $j=1,\dots,N$\;

            \textbf{Risk-averse target:}
            \[
            \Tc_k \leftarrow c_k + \alpha\bigl[
                (1-\varkappa)\tfrac{1}{N}\textstyle\sum_j V_{kj}^{\mathrm{dq}}
                +\varkappa\max_j V_{kj}^{\mathrm{dq}}\bigr]
            \]

        }

        $\Lc(\theta;\theta^-) \leftarrow \frac{1}{K}\sum_{k=1}^{K}
            \bigl[\Tc_k - Q_\theta(x_k,u_k^{\mathrm{macro}})\bigr]^2$\;
        $\theta \leftarrow \theta
            - \beta\,\nabla_\theta\Lc(\theta;\theta^-)$\;

    \If{$i \bmod \tau = 0$}{
        $\theta^-\leftarrow\theta$ \quad(hard update)
        \textbf{ or }\quad
        $\theta^-\leftarrow(1-\rho)\,\theta^-+\rho\,\theta$
        \quad(soft update)\;
    }
    }
}
\end{algorithm}

\section{Results}
\label{s:results}

\subsection{Comparison of Learning via Different Feature Sets}

Using the training methodology described in Section~\ref{s:training}, we simulate
$J_{\text{train}} = 8000$ episodes over randomly generated configurations to train the
neural network. Each episode is executed to completion without explicitly modeling
destructive events during training; instead, the probability of destruction is absorbed
into the discount factor, entering the learning objective in the standard way.
The environment parameters and neural-network hyperparameters are summarized in
Tables~\ref{tab:env_params}--\ref{tab:nn_params}.

\begin{table}[t]
\centering
\caption{Environment (configuration) parameters.}
\label{tab:env_params}
\begin{tabular}{lc}
\toprule
\textbf{Parameter} & \textbf{Value} \\
\midrule
Grid Size ($n \times n$) & $7 \times 7$ \\
Number of Collection Points & 12 \\
Number of Transmission Points & 2 \\
Number of Obstacles & 5 \\
Discount factor ($\alpha$) & 0.95 \\
Destruction probability ($\delta$) & 0.05 \\
\bottomrule
\end{tabular}
\end{table}

\begin{table}[t]
\centering
\caption{Training parameters for DDQN.}
\label{tab:training_params}
\begin{tabular}{lc}
\toprule
\textbf{Parameter} & \textbf{Value} \\
\midrule
Exploration Rate ($\epsilon$) & 0.3 \\
Number of Training Episodes & 8000 \\
Target Network Sync Frequency & 500 \\
Replay Memory Size & 6000 \\
Training Batch Size ($K$) & 800 \\
Risk Batch Size ($N$) & 2 \\
Risk Weight ($\varkappa$) & 1 \\
Validation Frequency & 50 \\
Learning Rate ($\beta$) & $1 \times 10^{-5}$ \\
\bottomrule
\end{tabular}
\end{table}

\begin{table}[t]
\centering
\caption{Neural network architecture.}
\label{tab:nn_params}
\begin{tabular}{lc}
\toprule
\textbf{Layer} & \textbf{Dimension} \\
\midrule
Input Layer ($l_1$) & $5n^2 + 1$ (raw) / $10$ (engineered) \\
Hidden Layer 1 ($l_2$) & 200 \\
Hidden Layer 2 ($l_3$) & 200 \\
Hidden Layer 3 ($l_4$) & 150 \\
Hidden Layer 4 ($l_5$) & 150 \\
Output Layer ($l_6$) & 2 \\
\bottomrule
\end{tabular}
\end{table}

To improve generalization, we investigate two feature-design strategies derived from the
underlying state space: (i)~raw features, consisting of the full state and configuration representation,
and (ii)~a curated set of task-relevant features selected to capture the most
informative aspects of the environment while eliminating redundancy.
Model performance while training is evaluated using training loss (to monitor convergence) and
cumulative reward on a held-out validation set of $J_{\text{val}} = 300$ independently
generated configurations, disjoint from the training data.

The results, summarized in Figure~\ref{fig:combined},
reveal a clear performance gap between the two representations.
The model trained on raw features achieves a lower \emph{training} loss but a
markedly worse reward on the validation set --- the classical signature of
overfitting to a high-dimensional and highly redundant encoding. The dimension
of the raw input, $5n^2+1$, grows quadratically with the linear size of the
grid, whereas the engineered input has a fixed dimension of $10$; and, as
observed in Section~\ref{s:training}, the raw encoding separates states that
belong to a single orbit of the symmetry group $G$, so that the network must
expend capacity on rediscovering an invariance that the engineered features
possess by construction.
The model trained on the curated features converges to a higher and visibly
more stable reward, confirming that the domain knowledge encoded in the feature
map reduces the effective complexity of the learning problem and improves
out-of-sample performance.
All subsequent experiments therefore adopt the curated feature representation.

\paragraph{Generalization Analysis}
We further evaluate the trained models on configurations outside the training
distribution by varying grid size, number of collection points, transmission nodes, 
and obstacles. The engineered-feature model consistently outperforms 
the raw-feature baseline across all settings, demonstrating strong scalability 
to problem instances with varying size and structural complexity (Tables~\ref{tab:configurations} 
and~\ref{tab:learning_comparison}).

Table~\ref{tab:configurations} summarizes the two evaluation environments. 
Configuration A corresponds to a smaller grid of size $7 \times 7$ with 
12 collection points, 2 transmission nodes, and 5 obstacles. Configuration B 
represents a more complex setting with a $12 \times 12$ grid, 14 collection points, 
3 transmission nodes, and 8 obstacles. Performance is evaluated over 
$J_{\text{test}} = 300$ randomly generated configurations that are disjoint 
from both training and validation sets.

\begin{figure}[h]
    \centering
    
    \begin{subfigure}{0.48\linewidth}
        \centering
        \includegraphics[width=\linewidth]{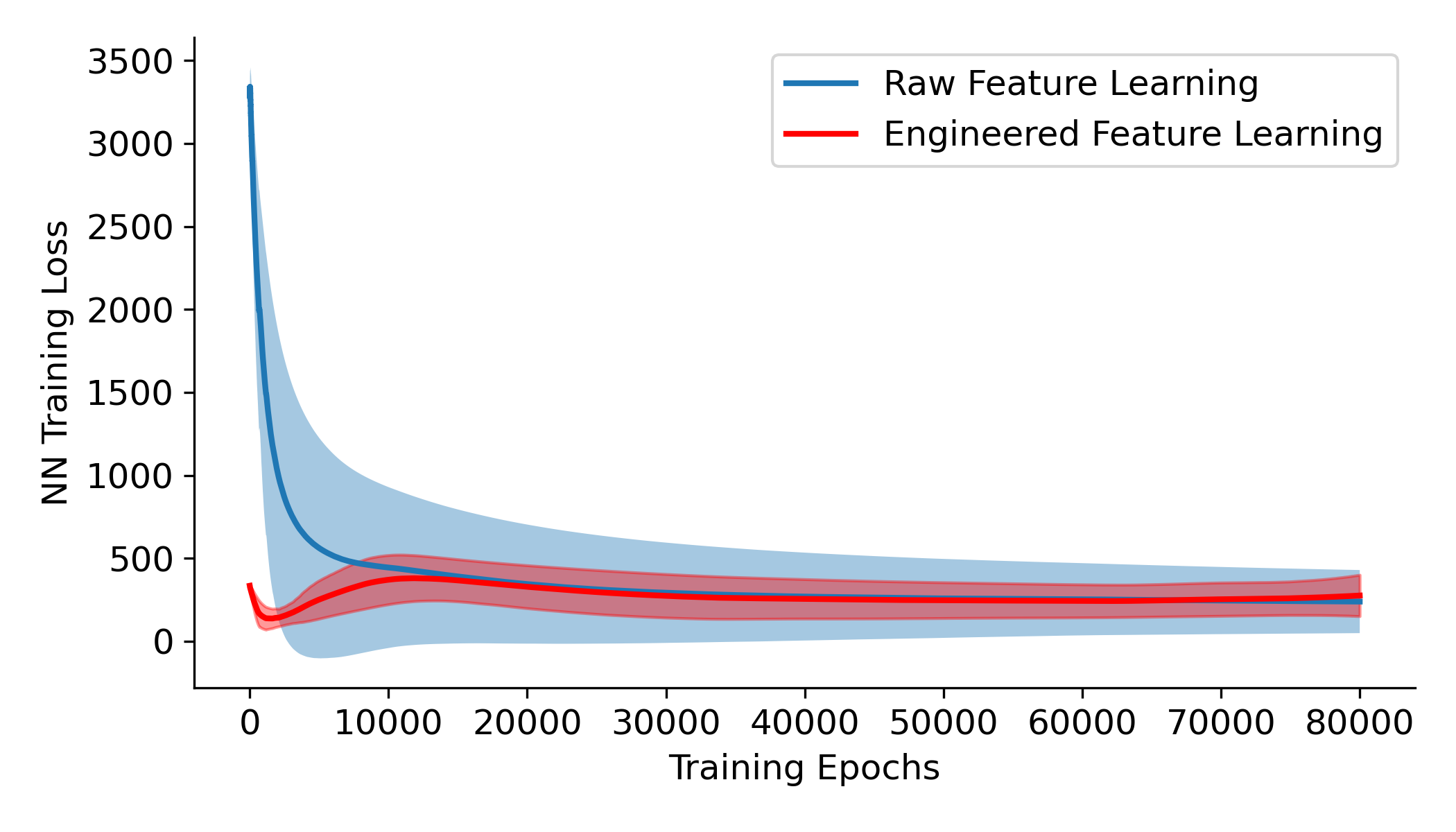}
        \caption{Training loss during learning}
        \label{fig:loss}
    \end{subfigure}
    \hfill
    \begin{subfigure}{0.48\linewidth}
        \centering
        \includegraphics[width=\linewidth]{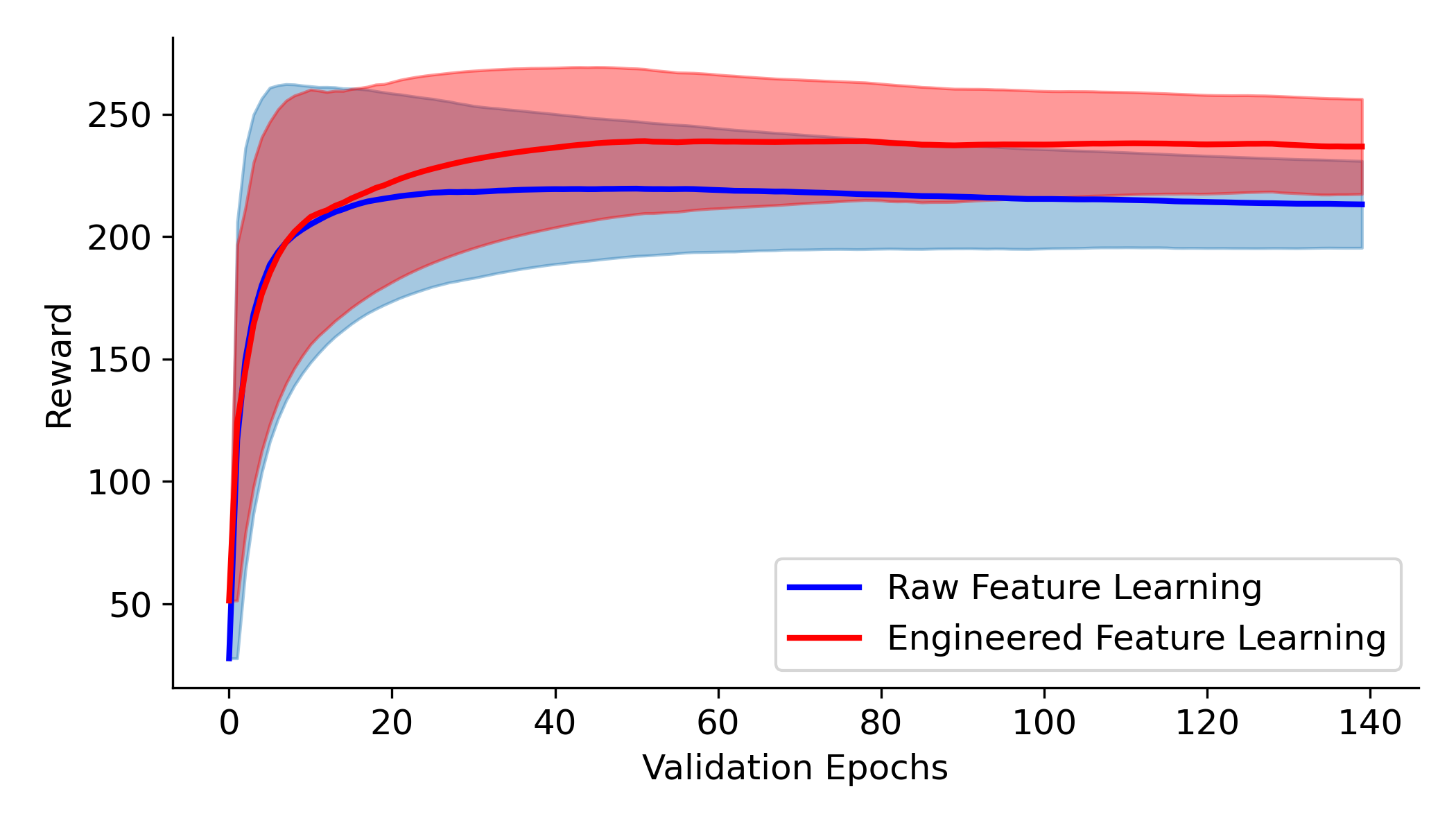}
        \caption{Cumulative reward learning during training}
        \label{fig:rewards}
    \end{subfigure}
    
    \caption{Comparison of optimization (Loss)  and performance (Reward) during training}
    \label{fig:combined}
\end{figure}

Here the {success ratio} is the fraction of test episodes in which the
robot reaches the terminal set $\Xc_T$, that is, visits every collection point
and delivers every payload.
As shown in Table~\ref{tab:learning_comparison}, the engineered-feature model
achieves higher mean rewards and attains a $100\%$ success ratio in both
configurations. The raw-feature model achieves a lower mean reward of $236.15$
and a success ratio of $95\%$ in configuration~A, and cannot be evaluated at all
on configuration~B: its input dimension $5n^2+1$ is tied to the grid size seen
in training, so the trained network does not even accept an instance of a
different size. This is not a quantitative deficiency but a structural one, and
it is precisely what the $G$-invariant, size-independent feature map of
Section~\ref{s:training} removes. The engineered representation transfers
without retraining from a $7\times 7$ grid with $12$ collection points to a
$12\times 12$ grid with $14$ collection points, retaining a $100\%$ success
ratio.

\begin{table}[t]
\centering
\caption{Environment configurations used for test experiments.}
\label{tab:configurations}
\begin{tabular}{lcc}
\toprule
\textbf{Configuration} & \textbf{A} & \textbf{B} \\
\midrule
Grid Size        & (7, 7)  & (12, 12) \\
Collection points   & 12      & 14 \\
Obstacles           & 5       & 8 \\
Transmission points & 2       & 3 \\
\bottomrule
\end{tabular}
\end{table}

\begin{table}[t]
\centering
\caption{Comparison of different feature learning methodologies under different configurations sampled on $J_{test} = 300$ in non crash setting}
\label{tab:learning_comparison}
\begin{tabular}{lcccc}
\toprule
\multirow{2}{*}{\textbf{Learning Methodology}} 
& \multicolumn{2}{c}{\textbf{Configuration A}} 
& \multicolumn{2}{c}{\textbf{Configuration B}} \\
\cmidrule(lr){2-3} \cmidrule(lr){4-5}
& Mean Reward & Success Ratio (\%) & Mean Reward & Success Ratio (\%) \\
\midrule
Raw Embedded & 236.15 & 95.33 & -- & -- \\
Engineered   & 259.56 & 100 & 184.74 & 100 \\
\bottomrule
\end{tabular}
\end{table}

\subsection{Risk-Averse Learning}

To model risk-averse behavior, we apply the mini-batch transition risk mapping
introduced in Section~\ref{sec:mini-batch}, which replaces the standard conditional
expectation in the Bellman operator with a coherent risk measure acting on
mini-batches of $N$ sampled transitions.
Setting $N = 1$ recovers the risk-neutral (expected-value) baseline, whereas
$N > 1$ induces increasing sensitivity to worst-case outcomes within each
mini-batch, interpolating between the risk-neutral policy and the worst-case
(robust) policy as $N \to \infty$.
We instantiate this framework with $N = 2$ and $\varkappa = 1$, and compare against two baselines:
(i)~the risk-neutral policy ($N = 1$), and (ii)~a highly conservative heuristic
policy parameterized by $\gamma = 2000$.

\begin{figure}[H]
    \centering

    \begin{subfigure}{0.45\linewidth}
        \centering
        \includegraphics[width=\linewidth]{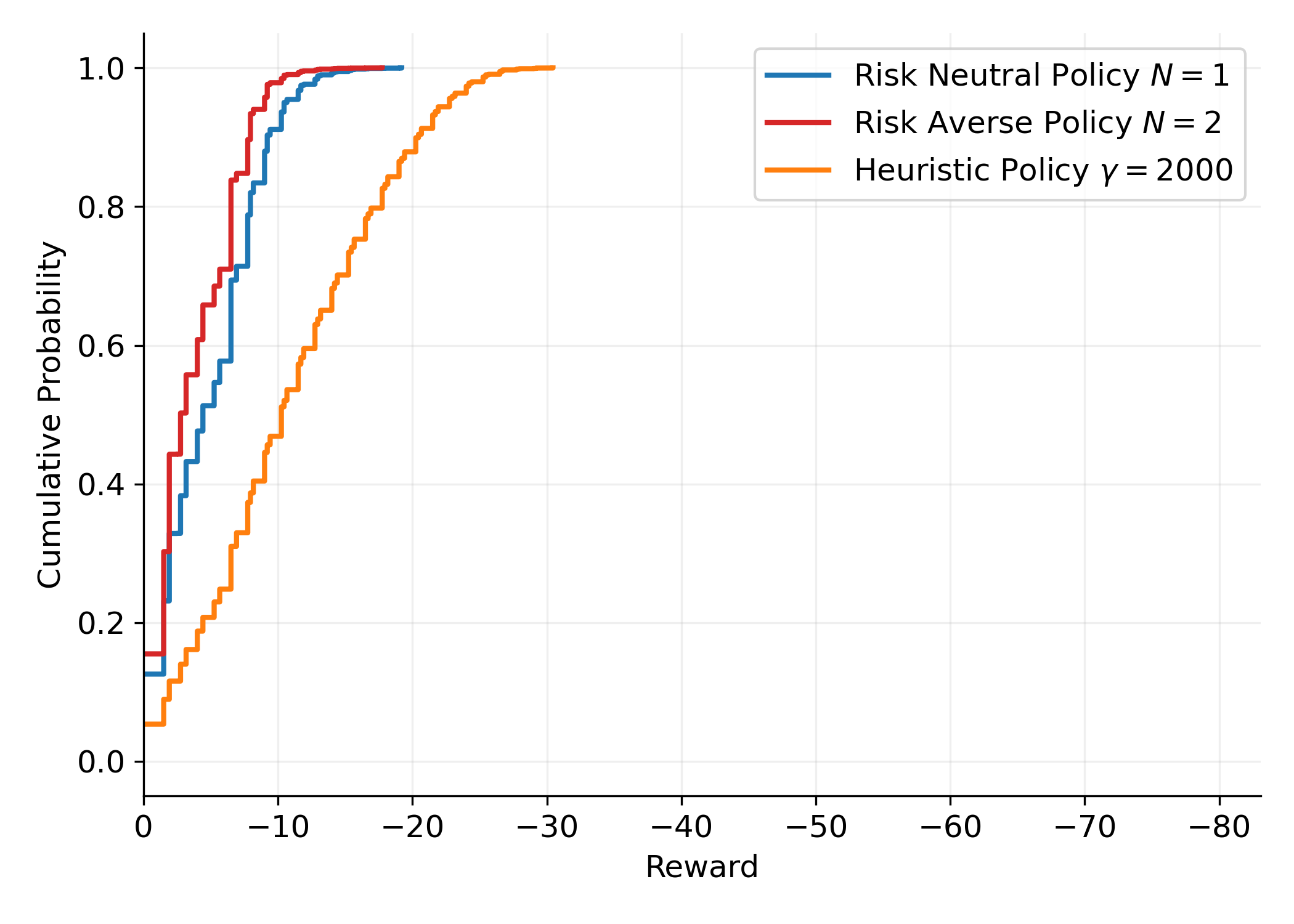}
        \caption{CDF plot of rewards (losses only) in non-crash}
        \label{fig:cdf_non_crash}
    \end{subfigure}
    \hfill
    \begin{subfigure}{0.45\linewidth}
        \centering
        \includegraphics[width=\linewidth]{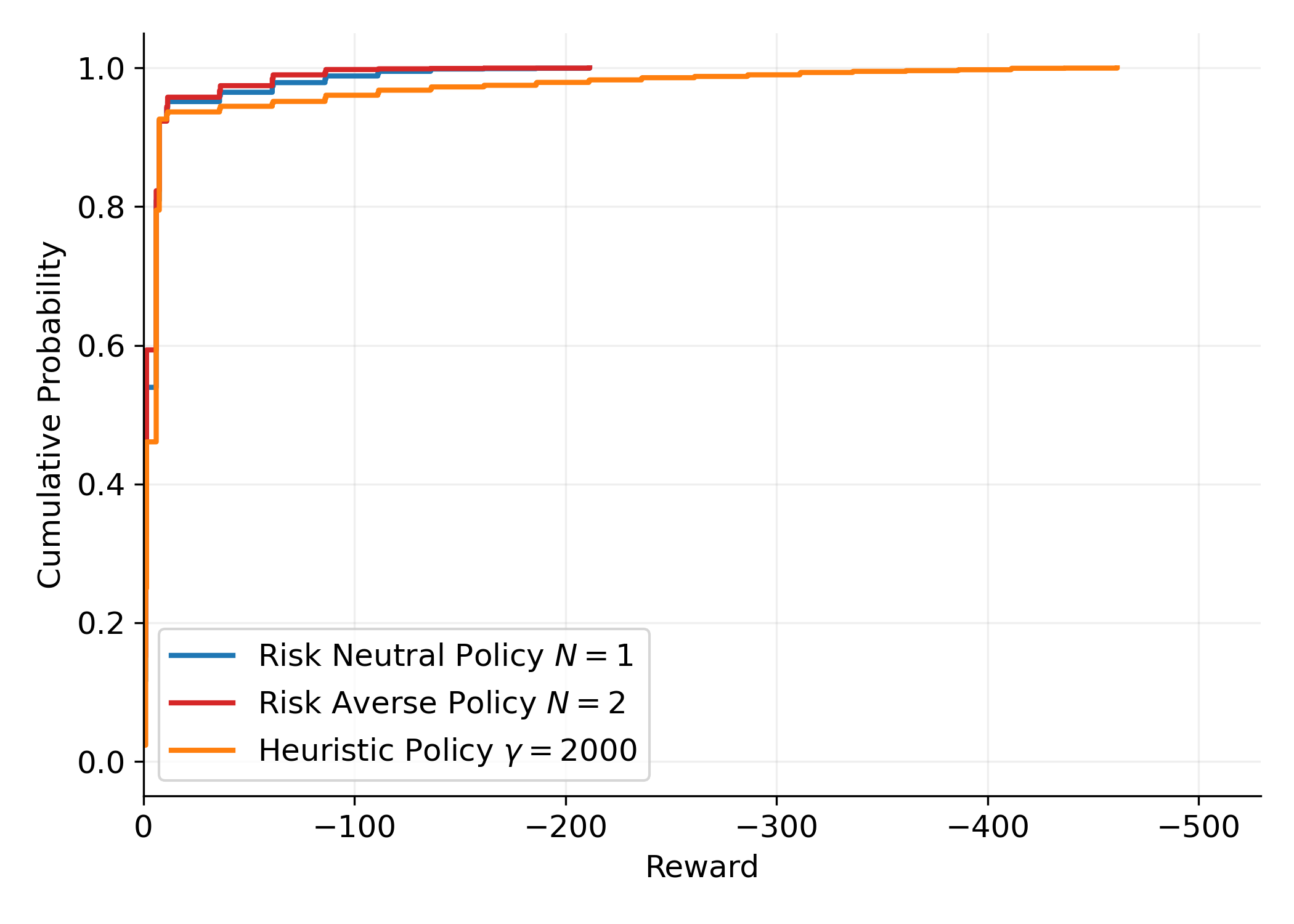}
        \caption{CDF plot with stochastic failure $p=0.05$}
        \label{fig:cdf_05}
    \end{subfigure}

    \vspace{0.5cm} 

    \begin{subfigure}{0.45\linewidth}
        \centering
        \includegraphics[width=\linewidth]{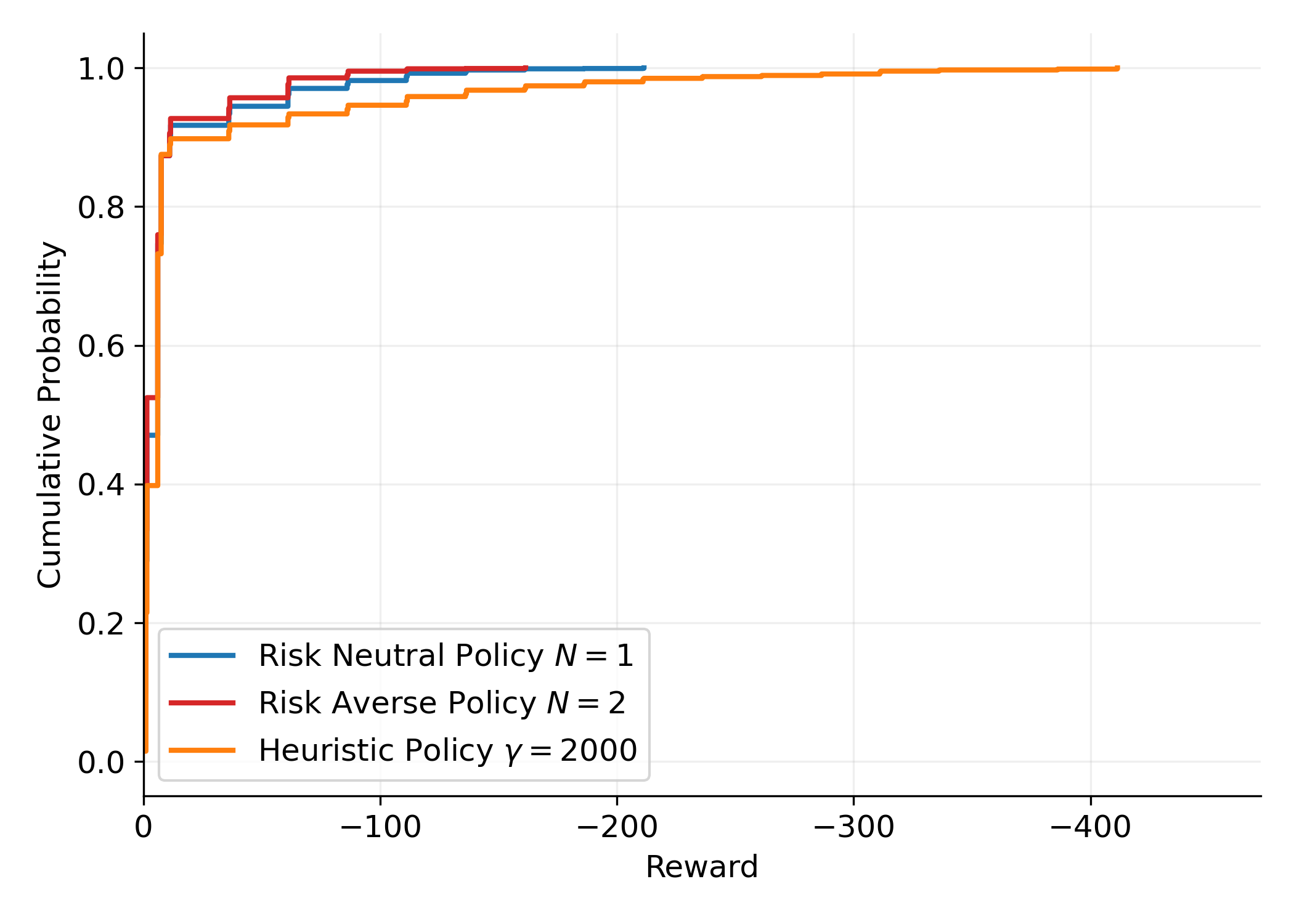}
        \caption{CDF plot with stochastic failure $p=0.10$}
        \label{fig:cdf_010}
    \end{subfigure}

    \caption{CDF comparison of different models in stochastic failure scenarios over 300 held-out environments}
    \label{fig:three_plots}
\end{figure}

\begin{table}[t]
\caption{Mean value and upper semideviation over 300 held-out environments:
$10\times10$ grid, 14~collection points, 3~transmitters, 8~obstacles.}
\label{tab:results-10x10}
\centering
\begin{tabular}{|c|cc|cc|cc|}
\hline
 & \multicolumn{2}{c|}{No crash}
 & \multicolumn{2}{c|}{Crash prob.\ $= 0.05$}
 & \multicolumn{2}{c|}{Crash prob.\ $= 0.10$} \\
\hline
Policy
& Mean Reward & $\sigma^{+}$
& Mean Reward & $\sigma^{+}$
& Mean Reward & $\sigma^{+}$ \\
\hline
$N = 1$
& $229.29$ & $742.44$
& $127.18$ & $211.12$
& $38.03$ & $122.26$ \\
\hline
$N = 2$
& $219.99$ & $719.13$
& $129.41$ & $186.22$
& $51.65$ & $111.70$ \\
\hline
$\gamma = 2000$
& $61.79$ & $805.35$
& $-66.66$ & $212.20$
& $-98.41$ & $118.58$ \\
\hline
\end{tabular}
\end{table}

Evaluation is conducted over 300 randomly generated environments with fixed seeds.
Each environment is defined on a $10 \times 10$ grid with 14~collection points,
3~transmission nodes, and 8~obstacles.
To assess robustness under stochastic failures, we introduce vehicle
destruction events with probabilities $\delta \in \{0.05,\, 0.10\}$, simulated
explicitly rather than absorbed into the discount factor
(cf.\ Remark~\ref{rem:killing}).
Performance is quantified by two statistics: the \emph{mean reward} and the
\emph{first-order upper semideviation} of the loss $L$,
\[
\sigma^{+}(L) = \Eb\big[\big(L - \Eb[L]\big)_{+}\big]
= \frac{1}{2}\Eb\big|L - \Eb[L]\big| ,
\]
which is used in risk-measure theory to capture the variability in the
unfavorable tail of the loss distribution~\cite{OgryczakRuszczynski1999}. It
is worth recalling from \eqref{msd-form} that $\Eb[L]+\kappa\,\sigma^+(L)$
is itself a coherent measure of risk for $\kappa\in[0,1]$, consistent with
second order dominance; $\sigma^+$ is therefore not an arbitrary dispersion
statistic, but the dispersion term of a legitimate mean--risk model, which is
what makes a comparison of policies in the $(\text{mean},\sigma^+)$ plane
meaningful.

Table~\ref{tab:results-10x10} demonstrates that the mini-batch policy with $N=2$ provides a particularly favorable risk--return profile across the different environment conditions. Under the nominal model, the $N=2$ policy achieves a mean reward of $219.99$, compared with $229.29$ for the risk-neutral policy ($N=1$), while simultaneously reducing the upper semideviation from $742.44$ to $719.13$. The reduction in mean performance is accompanied by a reduction in exposure to unfavorable outcomes attaining better mean--semideviation statistics for
$\varkappa > 0.4$. Under misspecification, the $N=2$ mini-batch policy exhibits a more favorable risk--return profile than the risk-neutral policy, improving both criteria simultaneously. At $\delta=0.05$, it increases the mean reward by $1.8\%$ while reducing $\sigma^{+}$ by $11.8\%$. At $\delta=0.10$, the advantage becomes more pronounced, with a $35.8\%$ improvement in mean reward accompanied by an $8.6\%$ reduction in $\sigma^{+}$.

This is not a paradox, and the mechanism is worth stating. Neither policy was
trained under explicitly simulated destruction; following
Remark~\ref{rem:killing}, the destruction probability was absorbed into the
discount factor, which is exact only for $N=1$. The crash scenarios are thus
perturbations of the training model, and the dual representation \eqref{dual}
says precisely that minimizing a coherent risk measure is minimizing a worst
case over an ambiguity set of transition kernels. The $N=2$ policy hedges
against a family of kernels which happens to contain the perturbed ones.  The mini-batch risk mapping supplies,
at the price of one additional next-state sample, the kind of protection that a
robust formulation would have to postulate explicitly.

The heuristic $\gamma$-policy, by contrast, is conservative without being
risk-averse in any principled sense. It degrades the mean severely,
without a corresponding reduction of the tail: its upper semideviation is the
largest of the three in the first two scenarios and essentially tied with the
risk-neutral one in the third. A single scalar threshold, calibrated by a line
search on one instance, cannot adapt to the risk--return structure of the
problem, whereas $\varkappa$ and $N$ act on the risk measure itself.
The empirical cumulative distribution functions in Figure~\ref{fig:three_plots}
corroborate these findings: relative to $N=1$, the $N=2$ policy shifts mass away
from the unfavorable tail in the crash scenarios, while the $\gamma$-policy
exhibits the heaviest unfavorable tail in the no-crash case.

\begin{figure}[h!]
    \centering
    \includegraphics[width=0.65\textwidth]{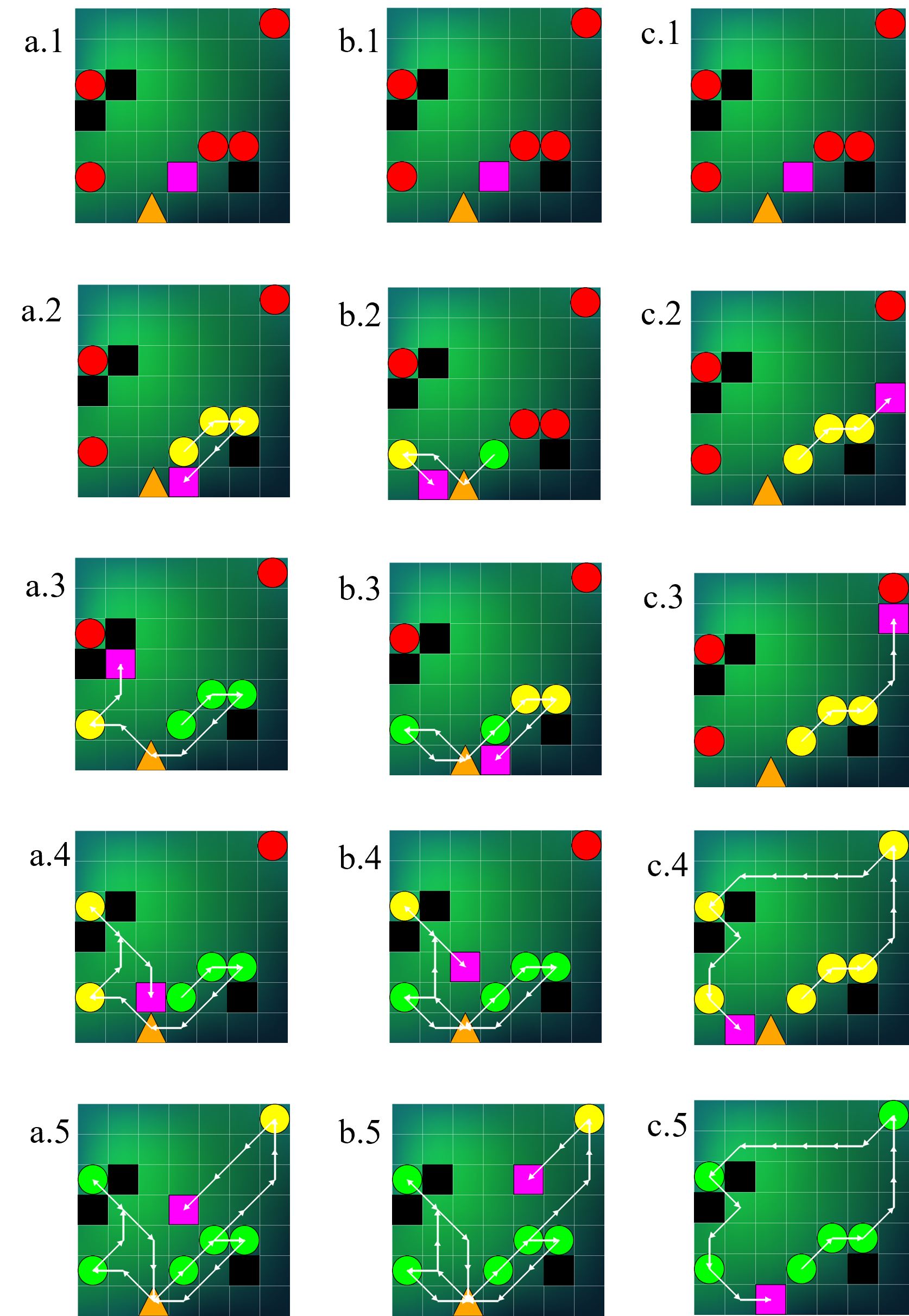}
    \caption{Comparison of trajectories under different trained policies a) Risk Neutral Policy ($N=1$) b) Risk Averse Policy ($N=2$) c) Heuristic $\gamma=2000$ Policy  }
    \label{fig:comaprison_trajectories}
\end{figure}

To illustrate the behavioral differences, Figure~\ref{fig:comaprison_trajectories}
presents representative trajectories for one fixed environment on a small
domain, in the absence of destruction events. Columns (a), (b), and (c)
correspond, respectively, to the risk-neutral policy ($N=1$), the risk-averse
policy ($N=2$), and the heuristic policy with $\gamma = 2000$, and the rows
show successive stages of the same episode. Moreover, the robot and the elements of the environment are described in Table~\ref{tab:environment_symbols}.

\begin{table}[ht]
    \centering
    \caption{Symbols used to represent the elements of the underwater robot environment.}
    \label{tab:environment_symbols}
    \renewcommand{\arraystretch}{1.4}
    \begin{tabular}{c l}
        \hline
        \textbf{Symbol} & \textbf{Description} \\
        \hline

        \textcolor{red}{\Large$\bullet$}
        & Unvisited collection point \\

        \textcolor{yellow}{\Large$\bullet$}
        & Visited collection point whose information has not been transmitted \\

        \textcolor{green}{\Large$\bullet$}
        & Collection point whose information has been successfully transmitted \\

        \textcolor{orange}{\Large$\blacktriangle$}
        & Transmission location \\

        \textcolor{black}{\Large$\blacksquare$}
        & Obstacle \\

        \textcolor{magenta}{\Large$\blacksquare$}
        & Robot \\

        \hline
    \end{tabular}
\end{table}

A comparison of the trajectories shows that the $N=2$ policy behaves more
conservatively: it avoids transporting a large accumulated payload along long
paths, and it transmits at the nearest available transmission point instead.
Under the cost structure of Section~\ref{s:costs}, the movement cost is
proportional to the carried payload, so a long journey with a full payload is
exactly the event that populates the unfavorable tail of the loss
distribution; the risk-averse policy shortens the exposure at the price of a
few additional visits to a transmission point.
The corresponding statistics, reported in
Table~\ref{tab:performance_comparison}, display the dominance:
the $N=2$ policy is giving $10\%$ more mean reward $136.008$ against $123.32$ ,
with the $4.5\%$ reduction of the upper semideviation $355.37$ against $339.18$, whereas the heuristic policy loses almost half of the mean
reward while \emph{increasing} the semideviation. These observations confirm
that the mini-batch risk mapping induces a principled and data-driven approach to the expected cost and the tail risk, rather than an undiscriminating
conservatism.

\begin{table}[h]
\centering
\caption{Reward and upper semideviation : 7 × 7 grid, 6 collection points,
1 transmitter, 3 obstacles.}
\label{tab:performance_comparison}
\begin{tabular}{lcc}
\toprule
\textbf{Policy} & \textbf{Reward} & \boldmath{$\sigma^{+}$} \\
\midrule
$N = 1$        & 123.324 & 355.37 \\
$N = 2$        & 136.008 & 339.18 \\
$\gamma = 2000$ & 62.910  & 375.41 \\
\bottomrule
\end{tabular}
\end{table}

\section{Conclusions}
\label{s:conclusions}

We have proposed a risk-averse deep Q-learning method for Markov decision
processes whose policies are evaluated by dynamic, time-consistent Markov risk
measures, and we have applied it to an underwater robot navigation problem of a
size that excludes tabular methods.

The construction rests on the mini-batch transition risk mappings of
\cite{dentcheva2023mini}. Their role here is to convert a nonlinear functional
of the transition kernel, which cannot be estimated from a single observed
transition, into an expected value of a function of $N$ next-state values,
which can. Proposition~\ref{prop:structure} shows that little is lost in the
process.  For $N=2$, mixed with the
conditional expectation, one obtains exactly the mean--Gini model. The batch
size is therefore a legitimate and interpretable dial of risk aversion, and the
smallest nontrivial choice $N=2$ already delivers a classical mean--risk model
at the cost of one additional next-state sample per transition.

On the algorithmic side, the combination with a double deep Q-network required
a separate analysis of two distinct biases. The maximum over the mini-batch is
not a source of bias at all: conditionally on the network parameters, it is an
exact evaluation of $\sigma^{(N)}$ at the current approximation, and this is
what makes the whole scheme implementable. The action-selection bias is the
familiar one and is removed by the double-network target; what remains is a
 convexity-induced bias of the opposite sign, which is
conservative and therefore benign in a risk-averse method.

Two devices proved decisive in the case study, and both are of a general
character. The first is the hierarchical decomposition, which delegates the
deterministic path-execution subproblem to an exact graph search and confines
the learning budget to the genuinely uncertain high-level decision. The second
is the invariant feature map: the model is invariant under the isometry group
of the grid, all the engineered features are invariants of that group, and none
of them depends on the size of the instance. This is what allows a single
trained network to be applied, without retraining, to instances larger than any
seen during training, while the raw encoding cannot even be evaluated on them.

The numerical results are the most interesting
outcome of this study. At the nominal model the risk-averse policy
loses a little to the risk-neutral one in the mean, and is slightly better in
the tail. Away from the nominal model, when destruction events are simulated
explicitly rather than absorbed into the discount factor, it dominates the
risk-neutral policy in \emph{both} criteria. This is the empirical shadow of the
dual representation \eqref{dual}: to minimize a coherent risk measure is to
minimize a worst case over a set of kernels, and that set is what protects the
policy against the perturbation. The heuristic threshold policy, by contrast,
is conservative without being risk averse, and pays a large price in the mean
without buying a reduction of the tail.

Several questions remain open. First, our method carries no convergence
guarantee. Convergence of risk-averse policy evaluation with a linear
architecture was established in \cite{kose2021risk,ruszczynski2025risk}, 
a reinforcement learning method with a linear architecture  was proposed and analyzed in \cite{ruszczynski2026reinforcement}, 
but
the extension to Q-learning with nonlinear approximators is open even in the
risk-neutral case, and the nonsmoothness of the mini-batch maximum adds a
difficulty of its own. Second, the hierarchical
scheme is formally a semi-Markov model, and the interaction of time consistency
with temporally extended actions (Remark~\ref{rem:smdp}) deserves separate
treatment: what is the right notion of a risk-averse option, and when is the
composition of one-step mappings along an option a mapping of the same type?
Third, we deliberately restricted attention to the smallest batch sizes, and a
systematic study of the trade-off among $N$, $\varkappa$, the statistical error,
and the computational cost would be valuable; Proposition~\ref{prop:structure}
suggests that $N$ and $\varkappa$ are not interchangeable, since they move the
spectrum in different ways. Finally, the destruction risk was absorbed into the
discount factor during training, which is exact only in the risk-neutral case
(Remark~\ref{rem:killing}); treating the killing event explicitly inside the
risk mapping is entirely feasible and, in view of our results under
misspecification, likely to be worthwhile.



\bibliographystyle{abbrv}


\end{document}